\documentclass[11pt]{article} 

\usepackage{geometry}
\usepackage{times}
\usepackage{graphicx}
\usepackage{algorithm}
\usepackage{algorithmic}
\usepackage{amsthm}
\usepackage{amsmath}
\usepackage{amssymb}
\usepackage{hyperref}

\newtheorem{theorem}{Theorem}
\newtheorem{example}{Example}
\newtheorem{lemma}{Lemma}
\newtheorem{proposition}{Proposition}
\newtheorem{observation}{Observation}
\newcommand{\eps}{\epsilon}
\newcommand{\E}{\textbf{E}}

\begin{document}

 \title{A Data Prism:\\ Semi-Verified Learning in the Small-$\alpha$ Regime}
\author{Michela Meister  \\ \emph{mmeister@stanford.edu} \and
    Gregory Valiant \\ \emph{valiant@stanford.edu}}

\maketitle

\begin{abstract}
We consider a simple model of unreliable or crowdsourced data where there is an underlying set of $n$ binary variables, each ``evaluator'' contributes a (possibly unreliable or adversarial) estimate of the values of some subset of $r$ of the variables, and the learner is given the true value of a \emph{constant} number of variables.   We show that, provided an $\alpha$-fraction of the evaluators are ``good'' (either correct, or with independent noise rate $p < 1/2$), then the true values of a $(1-\eps)$ fraction of the $n$ underlying variables can be deduced as long as $\alpha > 1/(2-2p)^r$.   For example, if each ``good'' worker evaluates a random set of $10$ items and there is no noise in their responses, then accurate recovery is possible provided  the fraction of good evaluators is larger than $1/1024$.  This result is optimal in that if $\alpha \le 1/(2-2p)^r,$ the large dataset can contain no information.   This setting can be viewed as an instance of the  \emph{semi-verified} learning model introduced in~\cite{CSV17}, which explores the tradeoff between the number of items evaluated by each worker and the fraction of ``good'' evaluators.  Our results require the number of evaluators to be extremely large, $ >n^r$, although our algorithm runs in \emph{linear} time, $O_{r,\eps}(n)$, given query access to the large dataset of evaluations.  This setting and results can also be viewed as examining a general class of semi-adversarial CSPs with a planted assignment.

This extreme parameter regime, where the fraction of reliable data is small (inverse exponential in the amount of data provided by each source), is relevant to a number of practical settings.  For example, settings where one has a large dataset of customer preferences, with each customer specifying preferences for a small (constant) number of items, and the goal is to ascertain the preferences of a specific demographic of interest.   Our results show that this large dataset (which lacks demographic information) can be leveraged together with the preferences of the demographic of interest for a \emph{constant} number of randomly selected items, to recover an accurate estimate of the entire set of preferences, even if the  fraction of the original dataset contributed by the demographic of interest is inverse exponential in the number of preferences supplied by each customer.   In this sense, our results can be viewed as a ``data prism'' allowing one to extract the behavior of specific cohorts from a large, mixed, dataset.
\end{abstract}


\section{Introduction}
Imagine that you have access to a large dataset of market research.  Specifically, the dataset consists of customer evaluations of products.  While the total set of products is large, of size $n$, each customer is only asked to evaluate a small (perhaps randomly selected) subset of $r=2,3,$ etc. of those products.   Long after the dataset is collected, suppose  you wish to identify the preferences of some special demographic of customers---perhaps  the customers who are full-time students.  Let $\alpha$ denote a lower bound on the fraction of the surveyed customers that were full-time students, but assume that we do not have this demographic information in our dataset--all we have is the set of evaluations of each customer.   How can we leverage this dataset to learn anything about the student-demographic?

If $\alpha \ll 1/2,$ this problem seems hopeless because the amount of data contributed by non-students might swamp the portion of the dataset contributed by the demographic of interest.  Nevertheless, the main result of this paper shows that one could hire some students to evaluate a \emph{constant}, $k$, number of (random) products in the set of size $n$, and then leverage that constant amount of information together with the large dataset to return accurate evaluations of the student-demographic preferences on all $n$ items.  This claim will hold provided the number of items evaluated by each of the customers in the dataset, $r > \log_2 (1/\alpha)$.  The guarantees of the algorithm will ensure that, with high probability, at most an $\eps$-fraction of the returned evaluations are incorrect (where $k$---the number of products evaluated by the hired students, is a function of $\eps$ that is independent of the total number of items, $n$).  In particular, this strong success guarantee holds irrespective of the behavior of the non-student demographics in the original dataset--in particular, they could even be adversarial, provided by a single malicious entity who is trying to disguise the feedback provided by the student-demographic.

The above setting, where one has a large dataset reflecting a number of demographics, and wishes to leverage the large dataset in conjunction with a very small set of ``verified'' datapoints from one demographic of interest, seems widely applicable beyond the market research domain.   Indeed, there are many biological or health-related datasets where the ``demographic of interest'' might be a trait that is expensive to evaluate.  For example, perhaps one has a large database of medical records, and wishes to investigate the propensity of certain medical conditions for the subset of people with a specific genetic mutation.  The large dataset of medical records will likely not contain information about whether individuals have the mutation in question.  Nevertheless, our results imply that accurate inferences about this subset of people can likely be made as long as 1) the fraction of people with the mutation in the large dataset is not minuscule, and 2) one can obtain a small (i.e. constant) amount of data from individuals with the genetic mutation in question, for example studying a constant number of individuals who are known to have the mutation.

\subsection{Formal Model}

We formally model this problem as an instance of the \emph{semi-verified learning} model proposed by Charikar, Steinhardt, and Valiant~\cite{CSV17}.   Suppose there is a set of $n$ Boolean variables, $V = \{v_1,\ldots,v_n\}$, and $m$ ``workers'' who each provide an evaluation of the values of a randomly selected subset of $r$ of the variables.  Suppose that an $\alpha$-fraction of the workers are ``reliable'' and submit evaluations with the property that each of their $r$ reported values is incorrect independently with probability $\le p_{rel}$.  We make no assumptions on the evaluations submitted by the $(1-\alpha)m$ unreliable workers---these evaluations could be biased, arbitrary, or even adversarially chosen with the goal of confounding the learning algorithm.   In addition to this large dataset, we also receive $k \ll n,m$ ``verified'' data points that consist of the values of a random subset of the variables of size $k.$   The goal of the learner will be to return assignments to the $n$ variables, such that with probability at least $1-\delta$, at most $\eps n$ of these returned assignments differ from their true values.

Previous work~\cite{steinhardt2016avoiding,CSV17} focussed on the regime where the number of workers, $m = \Theta(n)$.  In contrast, we will allow $m \gg n$, and focus on the interplay between the number of variables evaluated by each individual, $r$, and the fraction of reliable workers, $\alpha$.  Throughout, our positive results hold when the number of verified data points, $k$, is a constant that is independent of $n$, but dependent on $\eps, \delta,$ and $\alpha$.  

\subsection{Summary of Results and Connections to Random CSPs}\label{sec:sum}

Our main result is the following:
\begin{theorem}\label{thm:m}
Fix a failure probability $\delta>0$ and accuracy parameter $\eps > 0$.  Consider a set of $n$ items that each have a Boolean value, and $m$ reviewers who each evaluate a uniformly random subset of $r$ out of the $n$ items.  Suppose that $\alpha m$ of the reviewers are ``good'' in that each of their $r$ reviews is correct (independently) with probability at least $1-p\ge 1/2$.  Given sufficiently many reviewers, accurate reviews of at least $(1-\eps)n$ items can be inferred given the true values of a constant (independent of $n$) sized random subset of the variables, provided the fraction of good reviewers satisfies $\alpha >\frac{1}{(2-2p)^r}$.

Specifically, given the values of a random, constant-sized subset of the items of size $k=\tilde{O}\left(\frac{1}{\eps} \cdot 2^{2r} \log(1/\delta)\right)$, with probability at least $1-\delta$  one can recover accurate evaluations of at least $(1-\eps)n$ of the items, provided $\alpha > \frac{1}{(2- 2p)^r}$ and the number of reviewers $m= \tilde{\Theta}_{\alpha,\delta,\eps}(n^r)$.  

Additionally, the algorithm runs in time linear in the number of items, $n$, given the ability to query the dataset for reviewers who have evaluated a given set of items in constant time.  Specifically, the runtime of the algorithm is $O_{\delta,\eps,r}(n),$ where the hidden constant hides an exponential dependence on $r$, and polynomial dependence on $1/\eps$ and $\log(1/\delta).$
\end{theorem}

The following straightforward observation demonstrates that the above theorem is optimal in the relationship between the fraction of good reviewers, $\alpha$, and the number of items reviewed by each individual, $r$, and the error rate of each good reviewer, $p$:
\begin{observation}
If each good reviewer incorrectly reviews each item independently with probability $p$, and the fraction of good reviewers satisfies $\alpha = \frac{1}{(2-2p)^r}$ where $r$ denotes the number of items evaluated by each reviewer, then the remaining $(1-\alpha)$ fraction of reviewers can behave such that for every set of $r$ items, for a randomly selected reviewer, the distribution of reviews for those items will be uniform over the $2^r$ possible review vectors.  Hence the dataset contains no useful information.
\end{observation}

One reason why Theorem~\ref{thm:m} is surprising is that this inverse exponential dependence between the number of reviews per reviewers, $r$, and the fraction of ``good'' reviewers, can not be attained via the usual approach of low-rank matrix approximation that is often applied to this problem of recommendation systems (e.g.~\cite{candes2010matrix,keshavan2010matrix}).  To see why these approaches cannot be applied, note that for any matrix in which all rows have at most $r$ entries, there is a rank $r$ matrix that exactly agrees with all entries.  Intuitively, each of these $r$ factors is capable of representing a different subset of the reviewers.  Still, at best this would result in an algorithm that is capable of capturing $r$ different groups of reviewers; in other words, it seems extremely unlikely that such approaches could yield positive results in the setting where the fraction of ``good'' reviewers was less than $1/r$, in contrast to our results that allow this fraction to be $1/exp(r).$

\medskip

The setting of Theorem~\ref{thm:m} can be easily mapped into the language of a constraint satisfaction problem.    Given the evaluations of the reviewers, we build the constraint satisfaction problem by associating a Boolean variable to each of the $n$ items, and for every set of $r$ variables, we define the set of allowable assignments to those variables to include any of the $2^r$ review vectors that constitutes more than a $1/2^r$ fraction of the review vectors for the associated items.   (In other words, if at most a $1/2^r$ fraction of the reviewers who evaluated a given set of $r$ items submitted a vector of reviews $\sigma=(\sigma_1,\ldots,\sigma_r)$, then $\sigma$ is not an allowable assignment for those variables.)   The requirement that $\alpha > \frac{1}{(2-2p)^r}$ guarantees that, for every set of $r$ items, irrespective of the behaviors of the $(1-\alpha)$ fraction of bad reviewers, for a randomly selected reviewer, the probability that the $r$  reviews are all correct is strictly larger than $1/2^r$. Additionally, our requirement on the number of reviewers, $m$, ensures that with high probability (by elementary concentration bounds) for every set of $r$ items, there are sufficiently many reviewers assigned to that set of $r$ items, so as to ensure that the number of accurate ratings (provided by the good reviewers) exceeds a $1/2^r$ fraction of the overall reviews for that set of $r$ items.  Hence, with high probability, we obtain a constraint satisfaction problem such that for every set of $r$ variables 1) the correct assignment is in the set of allowable assignments, and 2) at least one of the $2^r$ possible assignments is disallowed.

Given this mapping from the review/evaluation setting to constraint satisfaction problems, Theorem~\ref{thm:m} will follow immediately from the following result concerning a class of adversarial constraint satisfaction problems:

\begin{theorem}\label{thm:mainCSP}
Consider a set of $n$ Boolean variables, and a planted assignment $\sigma\in \{0,1\}^n.$  Suppose that for each subset of $r$ variables, $t=\{v_1,\ldots,v_r\}$, there is a subset $C_t \subset \{0,1\}^r$ of assignments such that $|C_t| \le 2^r -1$ and the planted assignment $\sigma$ (restricted to the variables in $t$) is in set $C_t$.  Given the ability to query the planted assignment values for a constant number of variables chosen uniformly at random, the planted assignment can be recovered with up to $\eps n$ errors, for any constant $\eps>0$.  

Specifically, for any $\eps,\delta>0$,  after querying the values of $$k =\tilde{O}\left(\frac{1}{\eps} \cdot 2^{2r} \log(1/\delta)\right)$$ variables, with probability at least $1-\delta$ we can output an assignment $\sigma' \in \{0,1\}^n$ that differs from the planted assignment, $\sigma$, in at most $\eps n$ values.  Additionally, the algorithm will run in time $O_{r,\eps,\delta}(n).$
\end{theorem}

There is a simple $VC$-dimension argument together with a sphere-packing result of Haussler~\cite{haussler1995sphere} that yields a tighter information theoretic recovery result, yielding an analog of the above theorem with polynomial (rather than super-exponential) dependence on $r$.\footnote{We thank an anonymous reviewer of an early version of this paper for drawing our attention to this.}  Specifically, the number of verified assignments must be  $k = O(\frac{1}{\eps}\left(r\log(1/\eps)+\log(1/\delta)\right)$.  This $VC$-dimension approach, however, seems to yield an algorithm with runtime at least $n^r$, as opposed to the linear time algorithms of Theorems~\ref{thm:m} and~\ref{thm:mainCSP}.  For practical settings, having a linear-time algorithm seems quite important; that said, exploring this problem from an information theoretic perspective is also worthwhile.    One natural question is whether one can achieve a best-of-both-worlds: a near-linear time algorithm with a polynomial dependence on $r$.  We discuss this problem more in Section~\ref{sec:futureWork}.

\begin{proposition}\label{prop:vc}
As in Theorem~\ref{thm:mainCSP}, consider a set of $n$ Boolean variables, and a planted assignment $\sigma\in \{0,1\}^n.$  Suppose that for each subset of $r$ variables, $t=\{v_1,\ldots,v_r\}$, there is a subset $C_t \subset \{0,1\}^r$ of assignments such that $|C_t| \le 2^r -1$ and the planted assignment $\sigma$ (restricted to the variables in $t$) is consistent with $C_t$.   Given the ability to query the planted assignment values for $k=O\left(\frac{1}{\eps}\left(r \log(1/\eps)+\log(1/\delta)\right)\right)$ random entries, with probability at least $1-\eps$ one can recover an assignment that disagrees with $\sigma$ on at most $\eps n$ values.
\end{proposition}
\begin{proof}
Let $S\subset \{0,1\}^n$ be the set of assignments that are consistent with all of the sets of partial assignments to the $r$-tuples specified by the sets $C_t$.  The Vapnik-Chervonenkis (VC) dimension of the set $S$ is at most $r$, since, by assumption, for every $r$-tuple of variables,  $t=\{v_1,\ldots,v_r\}$, there are at most $|C_t| \le 2^{r}-1$ possible assignments to those variables.     As was shown by Haussler (Theorem 1 in~\cite{haussler1995sphere}), for any subset $S$ of the Boolean hypercube with VC dimension at most $r$, for every $\eps>0$ there exists a set $T \subset \{0,1\}^n$ of size at most $e(r+1)\left(\frac{2e}{\eps}\right)^r$ such that for every point $x \in S$, there exists a point $t_x \in T$ that agrees with $x$ on at least $(1-\eps)n$ coordinates.    

Let $T_{\eps/2}$ denote such a covering set corresponding to the set $S$, such that every $x \in S$ is distance at most $n\eps/2$ from an element of $T_{\eps/2}$.  We can use our $k = O\left(\frac{1}{\eps}\left(r \log(1/\eps)+\log(1/\delta)\right)\right)$ random coordinates of the vector $\sigma \in S$ to find, with probability at least $1-\delta$, a point in $T_{\eps/2}$ of distance at most $n\eps$ from $\sigma$ by simply choosing the element of $T_{\eps/2}$ that agrees with the largest fraction of the $k$ random samples.  This follows from 1) leveraging a Chernoff bound to show that out of the $k$ samples, at most a $(2/3) \eps$ fraction will disagree with the element of $T_{\eps/2}$ that has distance $\eps n /2$, and 2) a union bound over $|T_{\eps/2}|$ Chernoff bounds to argue that none of the elements of $T_{\eps/2}$ that have distance at least $\eps n$ will disagree in fewer than a $(2/3) \eps$ fraction of indices.  Together, this yields that the probability that the element of $T_{\eps/2}$ that agrees with the largest fraction of the $k$ random samples has distance greater than $\eps n$ from the true assignment, is at most $|T_{\eps/2}|exp(O(-k \eps)) = |T_{\eps/2}| (1/\eps)^{O(r)} \delta,$ which is at most $\delta$ for a suitable choice of the constant in the ``O'' term of $k=O\left(\frac{1}{\eps}\left(r \log(1/\eps)+\log(1/\delta)\right)\right)$.
\end{proof}

One implication of the above result is that for any Boolean constraint satisfaction problem for which 1) there exists a satisfying assignment, and 2) for every subset of $r$ variables the constraints forbid at least one of the $2^r$ possible assignments, it  must be the case that there are only a \emph{constant} number of ``$\eps$-similar solution clusters,'' where an $\eps$-similar solution cluster is a set of assignments that differ from each other in at most $\eps n$ locations.  Indeed, the number of such clusters will be at most $2^k$, where $k = \tilde{O}(r/\eps)$ is as specified in Theorem~\ref{thm:mainCSP} and Proposition~\ref{prop:vc}, is a bound on the number of variables whose assigned value must be queried to achieve a constant probability of failure $\delta < 1$.   Note that this number of solution clusters is independent of $n$.  

This structure of the satisfying assignments is slightly surprising given the following two simple examples: the first example illustrates that it is possible for such CSPs to have at least two extremely different satisfying assignments, and the second illustrates that it is possible for such CSPs to have super-constant sized solution clusters---clusters of size $\Omega(n)$---although all the assignments in such a cluster are quite similar. 

\begin{example}
Consider the setting where the underlying assignment to all $n$ variables is $T$, and for every pair of variables, the set of allowable assignments is $\{(F,F), (T,T)\}$.  Based on these constraints, there are two possible satisfying assignments---either all $T$ or all $F$.    A single ``verified'' data point is sufficient to distinguish between these two sets of assignments.
\end{example}

The following example illustrates that, in general, it is impossible to guarantee that the learner will correctly output the \emph{exact} assignment, unless the number of verified datapoints $k = \Theta(n)$.
\begin{example}
Consider the setting where each set of $r$ values has the constraint that precludes the $(F,F,\ldots,F)$ $r$-tuple.  In this case, there is a single solution cluster consisting of all assignments to the $n$ variables such that at most $r-1$ of the variables are $F$ and the remaining $n-r+1$ are $T$.  In this case, it is impossible to distinguish between these assignments with any significant probability using fewer than  $\Theta(n)$ verified evaluations.
\end{example} 

Despite the above examples, it is still unclear whether the information theoretic bound of Proposition~\ref{prop:vc} is tight; particularly for small constant $\eps$, it is not clear the extent to which the number of $\eps$-separated solution clusters can grow as $\eps$ decreases.

\subsection{Related Work}

Motivated by the increasing practical importance of robust estimation---and more generally, robust learning and optimization---there has been recent interest in these problems from both an information theoretic and computational perspective.   Recent works tackled this general problem in several basic settings, including robust linear regression~\cite{bhatia2015robust}, and robustly estimating the mean and covariance of natural classes of distribution, including multivariate Gaussians~\cite{diakonikolas2016robust,lai2016agnostic}.   The focus of these works was largely on establishing computationally efficiency algorithms for these tasks that approach the information theoretic (minimax) guarantees achieved by more naive or brute-force algorithms.  All three works focussed on the regime in which a majority of the data is assumed to be ``good''---drawn from the distribution or cohort of interest.  In the case of~\cite{bhatia2015robust}, the recovery guarantees require that this fraction of good data satisfies $\alpha \ge \frac{64}{65}.$

The recent works~\cite{steinhardt2016avoiding} and~\cite{CSV17} consider the setting where a minority of the data is ``good'' (i.e. $\alpha < 1/2$), with the latter paper formally proposing the ``semi-verified'' learning model where one may obtain a small amount of ``verified'' data that has been drawn from the distribution/cohort in question.  The former paper,~\cite{steinhardt2016avoiding} considers a similar item evaluation setting to the setting we consider, but focusses on the regime where the number of evaluators is on the same order as the number of items being evaluated.  In this regime, they show that $\eps$-accurate recovery is possible provided that the number of items reviewed by each evaluator is $O(\frac{1}{\eps^4\alpha^3})$ .  

In contrast, we consider the regime in which the number of evaluators might be significantly larger than the number of items, but establish an optimal tradeoff between the fraction of good reviewers and the number of items evaluated by each reviewer, demonstrating the surprising ability to tolerate a fraction of good evaluators that is inverse exponential in the number of items evaluated by each evaluator.   For the context of leveraging these techniques as a ``prism'' to extract information about specific demographics from a large, mixed dataset, this small-$\alpha$ regime seems especially significant.   The techniques of this paper, via local algorithms and the constraint-satisfaction perspective, also differ significantly from the previous approaches  to robust estimation which rely on more geometric or spectral structure.

The general challenge of developing algorithms and estimators that are robust to corruptions in the input data dates back to the early work of~\cite{tukey1960survey}, and has led to a significant body of work on  ``Robust Statistics'', which explores a number of different models for the data corruptions, and largely focusses on the regime in which a majority of the data is ``good.''   Much of this work is orthogonal to the objectives of this paper, and we refer the reader to the surveys of~\cite{huber2011robust} and~\cite{hampel2011robust}.

\section{The Algorithm}
In this section we describe a simplified algorithm that obtains the claimed result of Theorem~\ref{thm:mainCSP} with the exception of two key properties:  as-described, the runtime of this algorithm is $O(n^r)$ rather than $O_r(n)$, and this algorithm will require a number of ``verified'' samples that is inverse polynomial in the error parameter $\eps$, as opposed to the nearly inverse linear dependence specified in the theorem.   The algorithm to which Theorem~\ref{thm:mainCSP} applies is an  extension of this algorithm, and we describe it in Section~\ref{sec:realAlg}.

The overall structure of the algorithm is to reduce an instance of the problem with non-trivial constraints on all sets of $r$ variables, to an instance of the problem that has non-trivial constraints on all sets of $r-1$ variables.   In general, the true assignment might \emph{not} satisfy the constraints that we derive on the sets of $r-1$ variables, though we will be able to leverage any such derived constraints that are discovered to be false.   We begin by providing the intuition for the algorithm in the case that $r=2$, and then in Section~\ref{sec:intuitRed} describe the intuition for the reduction from constraints on sets of $r$ variables to constraints on $r-1$-tuples.  We formally describe the general algorithm in Section~\ref{sec:realAlg}.

\subsection{Intuition: Restricting to Pessimistic Constraints}

Our algorithm will proceed iteratively, with the goal of each iteration being to inspect at most a constant number of randomly sampled ``verified'' variable values, and return accurate guesses for at least a constant fraction of the variables.  The algorithm will then recursively iterate this procedure on the remaining variables until all but $\frac{\eps}{2}n$ variables have been assigned guesses; assignments to these last $\le \eps n/2$ variables can be chosen arbitrarily.

To begin, consider the setting where $r=2$, and for every pair of variables $(x,y)$ we have a set of allowable assignments, $C_{(x,y)} \subset \{T,F\}^2$, with $|C_{(x,y)}| \le 3.$    Each such set provides at least two implications, one of the form $x = X \implies y = Y$ and one of the form $y=Y' \rightarrow x = X'$ for some choice of  $X,X',Y,Y' \in \{T,F\}$.  For example, if the assignment $(T,F) \not \in C_{(x,y)},$ then we have the implications $x = T \implies Y = T$ and $y=F \implies x = F$.   In other words, there is at least one value of variable $x$ that would imply the value of variable $y$, and similarly for $y$.  

Hence, if we fix variable $x$, and consider the implications derived from the sets $C_{(x,y)}$ as $y$ ranges over all $n-1$ other variables, there must be an assignment to variable $x$ that would imply the values of at least $n/2$ variables.  We will refer to this assignment as the ``optimistic'' value of $x$, as this assignment to $x$ would immediately yield the values of at least half the remaining variables, and we would be done with the current iteration of the algorithm, and would then recurse on the remaining variables that have not been assigned values.

The first key idea of our algorithm is that we will assume that \emph{all} variables take their ``pessimistic'' values.   We will then ``check'' this assumption by revealing the true values of a random sample of $O(\log(1/\delta)/\eps)$ of these variables.  If all of these values are consistent with the ``pessimistic'' values, we can conclude that with probability at least $1-\delta,$ at least $(1-\eps) n$ of the variables actually take their ``pessimistic'' values, and hence we can simply output this assignment.   If, however, any of our $O(\log(1/\delta)/\eps)$ random checks fails, that means that we have found a variable that takes its ``optimistic'' value, and hence that one variable, $x$, together with the $n-1$ constraint sets $C_{(x,\cdot)}$ that involve it, imply the values of at least $n/2$ variables.  In either case, our constant (dependent on $\eps,\delta$) number of checks has yielded an accurate assignment to at least half the variables.  This simple algorithm in the $r=2$ case is summarized in the following pseudo-code:

\begin{algorithm}[H]
FindAssigments, $r=2$:\\
Input: Set of $n$ variables, and for every pair $(x,y)$, a set of allowable assignments to those variables $C_{(x,y)} \subset \{T,F\}^{2}$, with $|C_t| \le 3.$  Error parameter $\eps>0$ and failure parameter $\delta> 0$.\\
Output: Assignments to each of the $n$ variables.
\begin{itemize}
\item While there exists $\ge \eps n/2$ variables without assignments
\begin{itemize}
\item Let $n'$ denote the number of remaining variables, and for each of these, determine an ``optimistic'' assignment that would imply the values of at least $n'/2$ other variables, and define a variable's ``pessimistic'' value to be the opposite assignment. 
\item Consider a set of $10\frac{\log(1/\delta)}{\eps^2}$ randomly chosen variables and their ``verified'' assignments.  (If fewer than $\frac{\log(1/\delta)}{\eps}$ of these variables lie in the set of $n' > \eps n/2$ variables in consideration, output FAIL)
\item If all the verified assignments for variables in the set of $n'$ agree with their pessimistic assignments, then assign these $n'$ variables their pessimistic assignments.
\item Otherwise, we must have found a variable whose verified assignment is its \emph{optimistic} assignment, and we can assign the values of at least $n'/2$ variables accordingly.
\end{itemize}
\end{itemize}
\end{algorithm}

\subsection{From $r$-tuples to $r-1$-tuples: Pessimism All The Way Down}\label{sec:intuitRed} 

Given the algorithm for the $r=2$ case, which is successful provided every pair of variables has at least one forbidden assignment, the question is how to reduce the setting with constraints on sets of $r \ge 3$ variables, to the setting of constraints on sets of $r-1$ variables.   The following trivial lemma is the key to this reduction:

\begin{lemma}\label{lemma:eitherway}
Given an $r$-tuple and set of at most $2^{r}-1$ allowable assignments to those $r$ variables, then for any subset of $r-1$ of those variables, there exists an assignment to those $r-1$ variables that would imply the value of the $r$th variable.  
\end{lemma}
\begin{proof}
Consider a $r-1$ tuple, $t$, and an additional variable $v$, and the set of $\le 2^{r}-1$ allowable assignments to the $r$-tuple $(t \cup v)$.  If the restriction of these assignments to the $r-1$ variables in $t$ contains all $2^{r-1}$ possible assignments, it must be the case that for at least one of these assignments, there is a unique value that $v$ must assume, otherwise this $r$-tuple would have all $2^r$ possible assignments.   If the restriction of the $2^r - 1$ assignments to the $r-1$ tuple do not contain all $2^{r-1}$ assignments, then any such assignment would (vacuously) imply the value of the $r$th variable.
\end{proof}

The utility of this lemma is that if we have an $r-1$-tuple of variables, $t$, then by considering all possible additional variables $v \not \in t$, there exists an assignment to $t$ that determines the value of at least a $1/2^{r-1}$ fraction of the variables not in $t$.  Hence we can designate an ``optimistic'' assignment with the property that if that assignment holds, then it will imply assignments to at least a $1/2^{r-1}$ fraction of the remaining variables.   We will then assume that this ``optimistic'' assignment is \emph{not} allowed, thereby reducing the set of allowable assignments of variables in $t$ to size $2^{r-1}-1$, and proceed inductively.      In this sense, at some intermediate step of this algorithm where we are considering sets of $r' < r$ variables, the allowable sets of assignments that we are considering may not be completely accurate, as we are not verifying whether the sets actually do take their ``optimistic'' assignments or not.  However, if a $r'$-tuple of variables actually takes the values of a forbidden/optimistic assignment, then either it will immediately imply the values of a constant (i.e. at least $1/2^r$) fraction of variables, or it must be a subset of a larger tuple that takes its ``optimistic'' assignment.  Which of these two cases holds can be easily decided via querying the values of a (constant) number of random variables.  We describe the full algorithm in the following section.

\subsection{The Basic Algorithm}\label{sec:fullalg}

The high-level structure of the algorithm described in the previous part takes the form of a ``descending'' pass followed by an ``ascending'' pass.  In the descending pass,  we iteratively turn constraints on $r_0$ tuples into constraints on $r_0-1$ tuples, then $r_0-2$ tuples, etc; all the while, we forbid ``optimistic'' assignments to ensure that in the $r$th level, each $r$ tuple has at most $2^{r}-1$ allowable assignments.   This descending phase terminates with $r=1$, where we have our ``pessimistic'' conjectured assignments to all variables.  We then randomly check a few of these values; if we do not discover any inconsistencies with the conjectured values, then we can safely conclude that most of the conjectured values are correct.  

If we have discovered any inconsistencies, then we begin the ascending phase that investigates and checks any discovered ``optimistic'' assignments.  One minor wrinkle is that we should not trust the $>1/2^{r}$ fraction of values that appear to be implied by an  optimistic assignment to a set of $r < r_0-1$ variables.  These implications might be the result of forbidding an optimistic assignment for some larger tuple.  Nevertheless, if we randomly check some of the implications, then we will either verify the accuracy of these implications, or have found an optimistic assignment to a $r+1$ tuple.  In this sense, the ascending phase will either terminate upon satisfactorily verifying a significant (constant sized) subset of the set of output assignments, or we will have found an ``optimistic'' assignment to a  $r_0-1$ tuple, and the implications of $r_0-1$ tuples are based directly on the given set of constraints to $r_0$-tuples, which are valid by assumption.  Hence each phase of the algorithm will return assignments to a constant (at least $1/2^{r_0}$) fraction of the variables.

\begin{algorithm}
FindAssigments:\\
Input: Set of $n$ variables, integer $r_0$, and for every tuple $t \subset [n]^{r_0}$ of $r_0$ distinct variables, a set of allowable assignments to those variables $C_t \subset \{T,F\}^{r_0}$, with $|C_t| \le 2^{r_0} -1.$  Error parameter $\eps>0$ and failure parameter $\delta> 0$.\\
Output: Assignments to at least $(1-\frac{\eps}{2})n$ variables.
\begin{itemize}
\item While there exists $\ge \eps n/2$ variables without assignments, run DESCEND on the set of unassigned variables and their corresponding sets of allowable assignments.
\end{itemize}
\end{algorithm}

\begin{algorithm}[H]
DESCEND:\\
Input: For each $r$-tuple, a set $C$ of assignments to those variables, with $|C| \le 2^r -1.$
\begin{enumerate}
\vspace{-.2cm}\item If $r=1$, AscendAndVerify(set of assigned values to each variable, $r=1$).
\vspace{-.2cm}\item Else, for every $r-1$ tuple, $t$, we will create a set $C_t$ of $\le 2^{r-1}-1$ assignments:
\begin{itemize}
\vspace{-.2cm}\item Find an ``optimistic'' assignment $\sigma_t$ that would determine at least a $1/2^r$ fraction of variables not in $t$.  (The existence of such an assignment is guaranteed by Lemma~\ref{lemma:eitherway}.)
\item Set $C_t = \{T,F\}^{r-1} \setminus \sigma_t.$
\end{itemize}
\item Run DESCEND on the set of $r-1$-tuples and their corresponding sets of assignments, each of size $2^{r-1}-1$.
\end{enumerate}
\end{algorithm}

\vspace{-.4cm}\begin{algorithm}[H]
ASCEND AND VERIFY:\\
Input: Proposed assignments $\sigma_v$ for each variable $v \in V$ for some set $V$ of variables.  Integer $r$ indicating the size of the tuples whose constraints generated the proposed assignments, and assignment $\sigma_t$ to a $r-1$-tuple $t$, such that $C_{t \cup v}$ provided the implication $\sigma_t \rightarrow \sigma_v$.  Access to sets of allowable assignments corresponding to all tuples of size $r' \in \{r,\ldots,r_0\}$.  Constant $A = 2^{r_0}\log (1/\delta)\log(1/\eps))/\eps^2.$
\begin{enumerate}
\vspace{-.2cm}\item Randomly sample $A$ verified variable assignments.
\vspace{-.2cm}\item If all verified variable assignments agree with the proposed assignments, $\sigma_v$, then permanently assign $v\in V$ with their proposed assignments, $\sigma_v$.
\vspace{-.2cm}\item Otherwise, let $v$ denote a variable whose true/verified assignment $a_v \neq \sigma_v$, disagrees with the proposed assignment to $v$.  Hence $(\sigma_t,a_v) \not \in C_{t \cup v}$ so assignment $(\sigma_t,a_v)$ together with the constraints on the $r+1$ tuples must imply at least a $1/2^r$ fraction of variable assignments.  Denote these assignments by $\sigma^{new}$.
\begin{itemize}
\item Run  AscendAndVerify($\sigma^{new}, r+1,\{t \cup v\},(\sigma_t,a_v)$)
\end{itemize}
\end{enumerate}
\end{algorithm}

\section{An Efficient Algorithm}\label{sec:realAlg}

The linear-time variant of the basic algorithm described in the previous section hinges on two observations.  The first is that for a given $r$-tuple $t$, rather than consulting all $\Theta(n)$ constraints $C_{t \cup x_i}$ for all $x_i \not \in t$ to determine the ``optimistic'' assignment to $t$, one can determine an assignment that implies at least a $\frac{1}{2}\frac{1}{2^r}$ fraction of the variable values, with high probability, via \emph{sampling} a constant  (independent of $n$ but dependent on $r, \eps, \delta$) number of such constraints.  Note that this sampling does not look at any of the ``verified'' variable assignments---it just samples which of the constraints to consider.   We formalize this ability to efficiently determine an ``optimistic'' assignment via the following subroutine, and the following lemma characterizing its performance.

\vspace{-.4cm}\begin{algorithm}[H]
FIND OPTIMISTIC ASSIGNMENT:\\
Input: Set of $n$ variables $X$, $r$-tuple $t$, the ability to query constraints $C_{t'}$ for $|t'|=r_0 \ge r$ (i.e. the ability to find optimistic assignments to tuples $t'$ with $|t'| = r_0$) and probability of failure $\gamma > 0$.\\
Output: An optimistic assignment $\sigma_t$ to $t$ that would, with probability at least $1-\gamma$, imply the assignments to at least a $1/2^{r+1}$ fraction of other variables via the constraints $C_{t \cup x}$.    We define $C_t : = \{T,F\}^r \setminus \{\sigma_t\}.$
\begin{enumerate}
\item If $r=r_0$ then return constraint $C_t$.
\item Else
\begin{itemize}
\item Select $s = 3\cdot 2^{|t|} \log (1/\gamma)$ variables $x_1,\ldots, x_s$ uniformly at random from $X \setminus t$.
\item For each of these $s$ variables, $x_i$, compute $C_{t \cup x_i}$ via a (recursive) call to $FindOptimisticAssignment(X, t\cup x_i, ProbFailure = \gamma/(2s)).$
\item Define assignment $\sigma_t \in \{T,F\}^r$ to be the lexicographically first assignment that, via the constraints $\{C_{t \cup x_i} \},$ imply at least a $1/2^r$ fraction of variables $\{x_1,\ldots,x_s\}$.  [Note that such an assignment exists, since for each $x_i$,  $|C_{t \cup x_i}| \le 2^{r+1}-1$ has at least one out of the $2^r$ possible assignment to $t$ that would imply that value of $x_i$.]
\item Call $\sigma_t$ the ``optimistic'' assignment to tuple $t$, and store $C_t = \{T,F\}^r \setminus \sigma_t.$
\end{itemize}
\end{enumerate}
\end{algorithm}

The following two lemmas quantify the performance of the above algorithm.  The first lemma characterizes the probability of failure, and the proof follows immediately from standard Chernoff tail bounds.
\begin{lemma}
With probability at least $1-\gamma$  the optimistic assignment $\sigma_t$ returned by  algorithm FindOptimisticAssignment on input $X$ and $t$ has the property that for at least a $1/2^{|t|+1}$ fraction of variables $x \in X$, the assignment $\sigma_t$ together with the constraint set $C_{t \cup x}$ that would be computed by the algorithm on input tuple $t \cup x$, implies the value of variable $x$.  
\end{lemma}
\begin{proof}
Letting $p$ denote the true fraction of variables, $x$, whose assignments are implied by $\sigma_t$ and $C_{t \cup x}$.  Recall that $\sigma_t$ was chosen based on $s$ independent samples, yielding an empirical estimate $\hat{p} \ge 1/2^{|t|}$, and standard tail bounds yield that $\Pr[\hat{p} > 2p] \le e^{-\frac{s}{3\cdot 2^{|t|}}}$, yielding the lemma, since $s=3\cdot 2^{|t|} \log (1/\gamma)$.
\end{proof}

\begin{lemma}
Given constant-time query access to the constraint sets $C_{t'}$ for tuples satisfying $|t'|=r_0,$ for any tuple $t$, algorithm FindOptimisticAssignment on input $t$ and probability of failure $\gamma>0$ returns $C_t$ and runs in time/queries $\left(2\log(1/\gamma)\right)^{O(r_0^2)},$ which is independent of the size of the variable set, $|X|$.
\end{lemma}
\begin{proof}
Note that computing $C_t$ calls $O(2^{|t|} \log (1/\gamma))$ computations of $C_{t'}$ for $|t'| = |t|+1$, each called with error parameter $2^{|t|}$ smaller.  When $|t|=r_0$, $C_t$ is obtained via a single constant-time query.  Expanding this recursion yields the above lemma.
\end{proof}

The second observation that underpins the efficient algorithm is that we do not need to determine the optimistic assignments and form constraints $C_t$ for all ${n \choose r}$ $r$-tuples $t$.  For each phase of the algorithm, which returns assignments to a constant fraction of the unassigned variables---at least $1/2^{r_0}$---it suffices to find a single tuple $t$ that takes its ``optimistic'' assignment.  Indeed, such a tuple, by definition, takes values that imply assignments to a constant fraction of the remaining variables.  And for each of these variables, $x$, whose assignment is implied by the assignment to the tuple $t$, the value of variable $x$ can be determined in constant time by consulting the constraint $C_{t \cup x}.$  This observation is clarified in the following algorithm, which is an adaptation of the Descend/AscendAndVerify algorithm described in the previous section.   

Finally, we highlight the fact that the algorithm proceeds iteratively.  Given an initial set of variables, $Y$, at some intermediate step in the algorithm, we let $X$ denote the set of variables for which we have not yet output an assignment.   The algorithm will terminate when $|X| \le \eps |Y|/2$.  The goal of the current step of the algorithm will be to output assignments to at least a $1/2^{r_0}$ fraction of variables in $X$, such that the fraction of such assignments that are incorrect is bounded by $\frac{\eps}{2\log(2/\eps)} \frac{|Y|}{|X|}.$  Given this bound on the fraction of incorrect assignments returned at this phase of the algorithm, the total fraction of errors is  bounded by $\eps/2 + \int_{t=\eps/2}^1 \frac{\eps}{2\log(2/\eps)}\cdot \frac{1}{t} dt = \eps$  where the first $\eps/2$ is a bound on the error due to the arbitrary assignments to the last $\le \eps |Y|/2$ variables.    The benefit of having the target accuracy increase as $|X|$ decreases is because we are given verified samples, drawn uniformly at random from $Y$.  To ``check'' a proposed assignment to set $X$ to a target accuracy of $\gamma$, we need at least $1/\gamma$ verified samples from the set $X$  (ignoring the logarithmic dependence on the probability of failure).  To guarantee that this number of verified samples is obtained from set $X$, we will need to draw $\approx \frac{|Y|}{\gamma |X|}$   verified samples from $Y$.   Using the above trick of having the desired accuracy degrade as $|X|$ decreases, for each phase of the algorithm, a set of $\frac{|Y|/|X|}{\frac{\eps}{2\log(2/\eps)} \frac{|Y|}{|X|}} = \tilde{O}(1/\eps)$ verified samples is required---as opposed to the $\Theta(1/\eps^2)$ samples that would have been required if we had fixed the target error rate to be $\eps$ for all rounds of the algorithm.

\begin{algorithm}[H]
EFFICIENT FIND ASSIGNMENTS:\\
Input: Set of $n$ variables $Y$, integer $r_0$ and for every $r_0$-tuple $t \subset Y$, a set of allowable assignments $C_t$ with $|C_t| \le 2^{r_0}-1$.  Error parameter $\eps>0$ and probability of failure $\delta$. \\
Output: Set of $T/F$ assignments to each $x \in Y$.
\begin{itemize}
\item Set $T=  r_0\cdot 2^{r_0+1} \log (2/\eps)$.
\item While there are at least $\eps n/2$ unassigned variables:
\begin{enumerate}
\item Let $X \subset Y$ denote the set of unassigned variables.  
\item Let $\eps_X = \frac{\eps}{2\log(2/\eps)} \frac{|Y|}{|X|}$ denote the target accuracy of this round, and set $s = 10 \frac{|Y|}{\eps_X |X|} \log (10 T/\delta)$.
\item Take $s$ verified samples, revealing the planted assignment values for each of these variables.  Let $X_1\subset X$ denote the subset of these variables that are in set $X$, and for each $x \in X_1$ let $a_x$ denote the verified assignment to variable $x$.  If $|X_1| < s \frac{|X|}{2|Y|}$ output FAIL.
\item For each $x \in X_1$, determine $C_x$ via FindOptimisticAssignments with failure parameter $\gamma = \delta/T$.  
\item If, for all $x\in X_1,$ $a_x = C_x$, then for every variable $x \in X$, compute and output assignment $C_x$.
\item Otherwise, let $x_1 \in X_1$ denote a variable for which $a_x \neq C_x$, and run EfficientAscend$(X, i, x_1, a_{x_1}, s).$
\end{enumerate}
\end{itemize}
\end{algorithm}

\begin{algorithm}[H]
EFFICIENT ASCEND:\\
Input: Set of variables $X$, integer $i \in \{1,\ldots,r_0-1\}$, tuple $t$ with $|t|=i$, verified assignments $a_t \in \{T,F\}^{|t|}$ to tuple $t$, and parameter $s$.  \\
Output: Output to a subset of variables in set $X$.
\begin{enumerate}
\item If $i \ge r_0$ output FAIL.
\item Take $s_i = s\cdot 2^i$ verified samples, and let $X_{i+1}$ denote the intersection of $X$ with this set of variables with verified assignments, with $a_x$ denoting the verified assignment to variable $x \in X_{i+1}.$  
\item For each $x \in X_{i+1}$, determine $C_{t \cup x}$ via a call to FindOptimisticAssignment$(X,t \cup x, FailureProb=\delta/(10T\cdot s_i))$, and  let $X'_{i+1} \subset X_{i+1}$ denote the subset of variables $x \in X_{i+1}$ for which the  constraint $C_{t \cup x}$ together with $a_t$ implies a value $\sigma_x$ for $x$.   If $|X'_{i+1}| \le s_i \frac{|X|}{4 \cdot 2^{i} |Y|}$ output FAIL.
\item If, for all $x \in X'_{i+1},$ it holds that $\sigma_x = a_x$, then for every variable $x \in X$, compute $C_{t \cup x}$ and output assignment $\sigma_x$ if $\sigma_x$ is implied by $C_{t \cup x}$ and $a_t.$  
\item Otherwise, let $x_{i+1} \in X'_{i+1}$ denote a variable for which $C_{t \cup x_{i+1}}$ and $a_t$ implies assignment $\sigma_{x_{i+1}} \neq a_{x_{i+1}}$.  Run EfficientAscend$(X, i+1, t\cup x_{i+1}, (a_t,a_{x_{i+1}}),s).$
\end{enumerate}
\end{algorithm}

\begin{proposition}
Algorithm EfficientFindAssignments, when run with error parameter $\eps$ and probability of failure $\delta$, has the following properties:
\begin{itemize}
\item The algorithm will require at most $\tilde{O}\left(\frac{1}{\eps} \cdot 2^{2 r_0} \log(1/\delta)\right)$ verified samples drawn uniformly at random from the set of variables, $Y$.
\item With probability at least $1-\delta$, the algorithm will output assignments to each variable $x \in Y$, such that at most an $\eps$ fraction of the assignments disagree with the planted assignment.
\item  The algorithm runs in time $O_{r_0,\eps,\delta}(n)$, where the hidden constant is a function of $r_0,\eps,\delta$.
\end{itemize}\end{proposition}
\begin{proof}
The high level outline of the execution of algorithm EfficientFindAssignments is that in each step of the outer WHILE loop, an assignment to at least a $1/2^{r_0+1}$ fraction of the remaining unassigned variables, $X$, will be output.  This continues until $|X| \le \eps n /2$, at which point these remaining variables can be assigned arbitrary labels and the algorithm terminates.  Hence there will be at most $O\left(2^{r_0} \log(1/\eps)\right)$ iterations of the while loop.   In the iteration conducted on unassigned variable set $X \subset Y$, the goal will be to return assignments such that the fraction of returned assignments that are incorrect is at most $\frac{\eps}{2 \log(2/\eps)}\frac{n}{|X|}$, where $|Y|=n$ is the total number of initial variables.   Provided these accuracy goals are met at each step of the algorithm, the overall fraction of errors will be bounded by $\eps/2 + \int_{f=\eps/2}^1 \frac{1}{f} \cdot \frac{\eps}{2 \log(2/eps)} df = \eps,$ where the first term is the errors due to the arbitrary assignment to the remaining $\le \eps n/2$ variables.   Additionally, the number of verified samples required in each iteration is at most $O(r_0s\cdot s 2^{r_0})=\tilde{O}\left(2^{r_0} \frac{1}{\eps}\log(1/\delta)\right),$ hence the total number of verified samples across the $O\left(2^{r_0} \log(1/\eps)\right)$ iterations will be bounded by $\tilde{O}(2^{2r_0} \log(1/\delta)/\eps),$ as claimed.

We now analyze each run of the WHILE loop in EfficientFindAssignments, and the recursive calls to EfficientAscend.   At a high level, in each recursive call to EfficientAscend, either an assignment to at least a $1/2^{r_0+1}$ fraction of the remaining unassigned variables is returned via the implications from some (verified) optimistic assignment to a tuple, $t$; or, we have found a tuple $t \cup x_{i+1}$ for which we have verified assignments to all $|t|+1$ variables, and for which that assignment, $(a_t,a_{x_{i+1}}) \not \in C_{t \cup x_{i+1}}$ is the optimistic assignment, in which case the subsequent call to EfficientAscend considers this strictly larger tuple $t'=t \cup x_{i+1}$.   

To bound the runtime of the algorithm, note that each run of the algorithm requires constant time (dependent on $r_0,\eps,\delta$ but independent of the number of variables, $|X|$, up until the point in the algorithm when an assignment will be output (Step 4 in EfficientAscend).  At this point in the algorithm, at a computational expense of $O_{r_0,\eps,\delta}(|X|),$ an assignment to a constant fraction, at least $1/2^{r_0+1}$ of the remaining variables will be output, and the algorithm will then be repeated on the remaining unassigned variables.   Hence, the overall runtime of the algorithm will be linear in the number of variables.   

To bound the probability that a given run of the WHILE loop fails to successfully output an assignment to at least $|X|/2^{r_0+1}$ variables that meets the target accuracy of $\frac{\eps}{2 \log(2/\eps)}\frac{Y}{|X|}$, we will leverage a union bound over a number of standard Chernoff tail bounds.   First, note that the probability that EfficientFindAssignments outputs 'FAIL' in Step 3 in a given round of the algorithm is bounded by the probability that $|X_1| \le \E[|X_1|/2],$ where $X_1$ is a sum of i.i.d 0/1 random variables, hence this probability is bounded by $exp(-E[|X_1|]/8) \le \frac{\delta}{10 T}$, where $T,$ as specified in EfficientFindAssignments is a bound on the number of calls to EfficientAscend which bounds the number of runs of the WHILE loop.  Given that $|X_1| \ge \E[|X_1|/2],$ the probability that the assignment output in Step 5 of EfficientFindAssignments does not meet the target accuracy, $\eps_X=\frac{\eps}{2 \log(2/\eps)}\frac{Y}{|X|}$, is bounded by $(1-\eps_X)^{|X_1|} \le \frac{\delta}{10 T}.$  

The remaining  probability of failure stems from the execution of EfficientAscend.  In this algorithm, failure can stem from three different issues:  1) the constant number of constraints $C_{\cdot}$ computed via FindOptimisticAssignment prior to Step 4 of EfficientAscend can be erroneous and fail to imply the desired fraction of assignments.  The probability of this is bounded by $\delta/(10T s_i)$, which is sufficient to guarantee that \emph{every} optimistic assignment/constraint set $C_{\cdot}$ that is computed during the execution of the algorithm is accurate and implies the desired fraction of assignments, aside from the $O(|Y|)$ constraints computed during the assignment output steps---Step 4 of EfficientAscend and Step 5 of EfficientFindAssignments.   EfficientAscend will never output FAIL during Step 1, as the constraints corresponding to to $i=r_0$ are the constraints on $r_0$-tuples, which are satisfied by assumption.   The final potential failure mode of the algorithm is Step 3 of EfficientAscend, if the random set of verified assignments is insufficiently large to verify (to the target accuracy) a given potential set of assignments implied by an optimistic assignment via $C_{t \cup x}$.  Given that the assignment $a_t$ to tuple $t$ is optimistic, as guaranteed by the validity of FindOptimisticAssignments described above, this probability of failure is also a trivial application of standard Chernoff bounds, guaranteeing that the random variable $|X'_{i+1}|$ in Step 3 of EfficientAscend deviates from a lower bound on its expectation by at most a factor of 1/2.

A union bound over these probabilities of failure for each of the $\le T$ runs of the EfficientAscend algorithm yields the desired proposition.
\end{proof}

\section{Future Work}~\label{sec:futureWork}
This work shows that it is possible to tolerate a fraction of ``good'' data, $\alpha$, that is inverse exponential in $r_0$, the sparsity of each datapoint (i.e. the number of evaluations submitted per reviewer), provided the number of datapoints/reviewers is sufficiently large to ensure that each set of $r$ items has been evaluated by a significant number of good reviewers.  Our algorithm runs in time linear in the number of items to review (provided the ability to query summary statistics of the set of reviewers who have evaluated a given sets of items), and uses a constant number of ``verified'' reviews, which is independent of the total number of items to review, and depends inverse linearly on the desired error (to logarithmic factors).

One natural question is prompted by the results of~\cite{steinhardt2016avoiding}, which provide efficient algorithms for the regime where $r= poly(1/\alpha)$, but where the number of reviewers is linear in the number of items being reviewed (and uses a constant, dependent on $\alpha,\eps,\delta$ verified reviews).  Is it possible to achieve the best-of-both-worlds: $r=polylog(1/\alpha)$, and a number of reviewers that is linear, or grows significantly more slowly than the $n^r$ that we require, while leveraging a constant number of verified reviews?

To this end, our algorithm only ever considers ``single-hop'' implications of proposed assignments: an assignment to a set of $r$ variables is considered ``optimistic'' if it directly implies values for a significant fraction of the other variables.  It is easy to imagine extending this definition to also consider longer chains of implication.  Perhaps a specific assignment to $r$ variables would imply values to $c_1$ additional variables, which in turn would imply values to $c_2$ variables, etc.   Indeed, in the basic setting of $r=2$, this approach can be realized to yield an algorithm that only requires constraints on a random subset of size $O(n^{3/2})$, as opposed to the $O(n^2)$ constraints assumed in this work.

From a computational perspective, it seems unlikely that such an approach could be pushed to yield an efficient algorithm for the regime in which fewer than $n^{r/2}$ sets of $r$ variables have nontrivial constraints.  Indeed, even for random instances of $r-SAT$ with a planted solution, efficient algorithms below this threshold have been elusive (see, for example, the recent related work on random CSPs with planted assignments~\cite{feldman2015complexity,raghavendra2016strongly}).   

From a purely information theoretic perspective---the picture is not entirely clear either.  In contrast to random CSPs, our setting is complicated by the adversarial nature of the constraints that are placed on the $r$-tuples.   Even for a semi-adversarial setting where tuples are chosen at random, and an adversary chooses the constraints to place on those random tuples, it is not immediately clear how to analyze the extent to which implications ``propagate''.  A second difficulty is that the goal of our setting is not just to find a satisfying assignment, but to find something close to a specific planted assignment.  Our results imply, for the setting we consider, that there are at most a constant number of solution clusters.  It seems interesting to investigate  the extent to which this holds for semi-adversarial CSPs  with fewer constraints, perhaps with $n^{r/2}$ constraints being the threshold between a constant and super-constant number of such clusters.

 \bibliographystyle{plain}
\bibliography{bibfile,experiments}

\begin{thebibliography}{10}

\bibitem{bhatia2015robust}
Kush Bhatia, Prateek Jain, and Purushottam Kar.
\newblock Robust regression via hard thresholding.
\newblock In {\em Advances in Neural Information Processing Systems}, pages
  721--729, 2015.

\bibitem{candes2010matrix}
Emmanuel~J Candes and Yaniv Plan.
\newblock Matrix completion with noise.
\newblock {\em Proceedings of the IEEE}, 98(6):925--936, 2010.

\bibitem{CSV17}
M.~Charikar, J.~Steinhardt, and G.~Valiant.
\newblock Learning from untrusted data.
\newblock In {\em Symposium on Theory of Computing (to appear)}, 2017.

\bibitem{diakonikolas2016robust}
Ilias Diakonikolas, Gautam Kamath, Daniel~M Kane, Jerry Li, Ankur Moitra, and
  Alistair Stewart.
\newblock Robust estimators in high dimensions without the computational
  intractability.
\newblock In {\em Foundations of Computer Science (FOCS), 2016 IEEE 57th Annual
  Symposium on}, pages 655--664. IEEE, 2016.

\bibitem{feldman2015complexity}
Vitaly Feldman, Will Perkins, and Santosh Vempala.
\newblock On the complexity of random satisfiability problems with planted
  solutions.
\newblock In {\em Proceedings of the Forty-Seventh Annual ACM on Symposium on
  Theory of Computing}, pages 77--86. ACM, 2015.

\bibitem{hampel2011robust}
Frank~R Hampel, Elvezio~M Ronchetti, Peter~J Rousseeuw, and Werner~A Stahel.
\newblock {\em Robust statistics: the approach based on influence functions},
  volume 114.
\newblock John Wiley \& Sons, 2011.

\bibitem{haussler1995sphere}
David Haussler.
\newblock Sphere packing numbers for subsets of the boolean n-cube with bounded
  vapnik-chervonenkis dimension.
\newblock {\em Journal of Combinatorial Theory, Series A}, 69(2):217--232,
  1995.

\bibitem{huber2011robust}
Peter~J Huber.
\newblock {\em Robust statistics}.
\newblock Springer, 2011.

\bibitem{keshavan2010matrix}
Raghunandan~H Keshavan, Andrea Montanari, and Sewoong Oh.
\newblock Matrix completion from a few entries.
\newblock {\em IEEE Transactions on Information Theory}, 56(6):2980--2998,
  2010.

\bibitem{lai2016agnostic}
Kevin~A Lai, Anup~B Rao, and Santosh Vempala.
\newblock Agnostic estimation of mean and covariance.
\newblock In {\em Foundations of Computer Science (FOCS), 2016 IEEE 57th Annual
  Symposium on}, pages 665--674. IEEE, 2016.

\bibitem{raghavendra2016strongly}
Prasad Raghavendra, Satish Rao, and Tselil Schramm.
\newblock Strongly refuting random csps below the spectral threshold.
\newblock {\em arXiv preprint arXiv:1605.00058}, 2016.

\bibitem{steinhardt2016avoiding}
Jacob Steinhardt, Gregory Valiant, and Moses Charikar.
\newblock Avoiding imposters and delinquents: Adversarial crowdsourcing and
  peer prediction.
\newblock In {\em Advances in Neural Information Processing Systems}, pages
  4439--4447, 2016.

\bibitem{tukey1960survey}
John~W Tukey.
\newblock A survey of sampling from contaminated distributions.
\newblock {\em Contributions to probability and statistics}, 2:448--485, 1960.

\end{thebibliography}

\end{document}